\documentclass[11pt]{article}
\usepackage{amsmath}
\usepackage{amsthm}
\usepackage{amscd,amssymb,stmaryrd}
\usepackage{graphicx}
\usepackage{subcaption}
\usepackage{algpseudocode}
\usepackage{algorithmicx}
\usepackage{algorithm}

\newtheorem{theorem}{Theorem}[section]

\newtheorem{remark}[theorem]{Remark}

\title{An Energy-Based Self-Adaptive Learning Rate for Stochastic Gradient Descent: Enhancing Unconstrained Optimization with VAV method}
\author{Jiahao Zhang\footnote{Department of Mathematics, Purdue University, West Lafayette, IN 47907, USA. (Email: zhan2296@purdue.edu)}, 
Christian Moya\footnote{Department of Mathematics, Purdue University, West Lafayette, IN 47907, USA. (Email: cmoyacal@purdue.edu)},
and Guang Lin\footnote{Department of Mathematics and Mechanical Engineering, Purdue University, West Lafayette, IN 47907, USA. (Email: guanglin@purdue.edu)}
}
\date{}
\begin{document}
\maketitle
\begin{abstract}
Optimizing the learning rate remains a critical challenge in machine learning, essential for achieving model stability and efficient convergence. The Vector Auxiliary Variable (VAV) algorithm introduces a novel energy-based self-adjustable learning rate optimization method designed for unconstrained optimization problems. It incorporates an auxiliary variable $r$ to facilitate efficient energy approximation without backtracking while adhering to the unconditional energy dissipation law. Notably, VAV demonstrates superior stability with larger learning rates and achieves faster convergence in the early stage of the training process. Comparative analyses demonstrate that VAV outperforms Stochastic Gradient Descent (SGD) across various tasks. This paper also provides rigorous proof of the energy dissipation law and establishes the convergence of the algorithm under reasonable assumptions. Additionally, $r$ acts as an empirical lower bound of the training loss in practice, offering a novel scheduling approach that further enhances algorithm performance.
\end{abstract}

\section{Introduction}\label{sec:Introduction}
Optimization lies at the core of machine learning, providing the foundational methods for training models to make accurate predictions and decisions. Among the numerous challenges faced in optimization, determining the optimal learning rate is paramount. Traditional fixed learning rate methods often struggle to balance between rapid convergence and the risk of overshooting minima, necessitating frequent manual adjustments based on empirical observations. In response to these challenges, this paper introduces a novel self-adjustable learning rate optimization method designed to dynamically adapt the learning rate based on the evolving state of the model during training. the proposed method utilizes real-time feedback from the training process itself to adjust the learning rate. This adaptability not only accelerates convergence but also enhances the stability and performance of the learning algorithm across various machine-learning tasks.
\subsection{Problem Setting} \label{sub-sec:Notation}
Consider the unconstrained optimization problem:
\begin{equation}\label{minP}
    \min_{x \in \mathbb{R}^n}~f(x),
\end{equation}
where $f:\mathbb{R}^n \to \mathbb{R} $ is continuously differentiable and bounded from below, i.e., $f^* = \inf_{x \in \mathbb{R}^d}f(x) > -c.$ Thanks to the foundational contributions in \cite{brown1989some, saupe1989discrete, zufiria1990application}, it is established that applying a gradient descent method
\begin{equation}
    x_{n+1} = x_n - \eta\nabla f(x(t)),
\end{equation}
to the optimization problem \eqref{minP} can be interpreted as a numerical approach to solving the ordinary differential equation:
\begin{equation}
    \frac{dx}{dt} = - \nabla f(x(t)),
\end{equation}
In this context, $f(x)$ represents the loss function of a neural network, with $x$ denoting the network's parameters, including weights and biases. 

In real-world deep learning scenarios, training on the full dataset is not practical due to memory constraints. Consequently, this paper focuses on the stochastic (mini-batch) application of the proposed method, which we formalize as: 
\begin{equation}\label{iterAlm}
    x_{n+1} = x_n - \eta \nabla g_n
\end{equation}
Here, $g_n$ is the gradient of $f$ calculated over a selected subset of samples.
\subsection{Previous Works} \label{sub-sec:Energy-Methods}
The scalar auxiliary variable~(SAV) method was originally established in \cite{shen2018convergence, shen2018scalar, shen2019new} as a numerical scheme for solving a class of ODEs known as gradient flows. The application of SAV method in neural networks was demonstrated in many works such as \cite{zhang2024energy}. It was then adapted for optimization problems by \cite{liu2020aegd} through AEGD method. later, \cite{liu2023efficient, zhang2022numerical} introduced a relaxed version of the RSAV method optimized for full-batch training, enhancing its alignment with real-time loss values. Alongside these, other learning rate adaptivity methods such as \cite{zhao2022stochastic, defazio2023learningratefree, zeiler2012adadelta, kingma2017adam} continue to evolve, offering varied approaches to optimizing training processes in diverse computational environments.
\subsection{Our Contribution} \label{sub-sec:Previous-Works}

In this work, we present several novel contributions to the optimization algorithms:
\begin{itemize}
    \item \textbf{Novel Optimization Method:} We propose a novel self-adjustable learning rate, gradient-based optimization method. This approach demonstrates greater stability with larger learning rates compared to traditional Stochastic Gradient Descent (SGD), and it achieves faster convergence and enhanced performance across various machine learning tasks.
    \item \textbf{Theoretical Foundation:} We provide a rigorous proof of the dissipative law of the parameter 'r' and establish the convergence of the algorithm. This theoretical underpinning ensures the reliability and predictability of our method.
    \item \textbf{Energy-based Self-Adaptive Scheduler:} The 'r' parameter introduced in our algorithm serves as a critical indicator for monitoring the training process. In practice, it can be considered as the empirical lower bound of the training loss. It measures the distance to the optimal minimum, acting as a dynamic scheduler for the learning rate. This functionality allows for further acceleration of the convergence process, enhancing the efficiency of the training phase.
\end{itemize}
\section{VAV method} \label{sec:Preliminary}
A fundamental challenge in optimization problems is to measure the distance between the current state and the global minimum. Traditionally, this distance is quantified using the loss function, which is expected to decrease progressively over time. However, due to the non-convex nature of the problem and the application of stochastic methods, the training loss often exhibits non-dissipative behavior. To address this, a new variable $r$ is designed to monitor the training progress while preserving dissipative properties.

\begin{subequations}\label{equation:SAV}
\begin{align} 
\frac{x_{t+1} - x_t}{\eta} & + \frac{ r_{t+1}}{\sqrt{f(x_t;\xi_t)}} G_t = 0,  \\ \frac{r_{t+1} - r_t}{\eta} &= \frac{1}{2 \sqrt{f(x_t;\xi_t)}} \left\langle G_t,\frac{x_{t+1} - x_t}{\eta}\right\rangle \end{align}
\end{subequations}

In this framework, $r$ serves as an approximation of the loss value at the subsequent time step, suggesting an inherent relationship with the current loss. Previous studies \cite{jiang2022improving} have demonstrated that defining $r_t$ as a convex linear combination of $\hat{r_t}$ and $\sqrt{f(x_t; \xi_t)}$, under careful choice of coefficient $\omega_0$, preserves the energy dissipative property. This approach aligns $r_t$ closely with the dynamics of the loss function, facilitating an accurate tracking of its trajectory through the optimization landscape.

\begin{subequations}\label{equation:RSAV}
\begin{align} 
&r_{t+1} = \omega_0 \tilde{r}_{t+1} + (1-\omega_0) \sqrt{f(x_{t+1};\xi_{t+1})}
\end{align}
\end{subequations}

where, for a given $\psi \in (0,1),~\omega_0$ is the smallest number in the interval $[0,1]$ such that
$$(r_{t+1})^2 - (\tilde{r}_{t+1})^2 \le \frac{\psi}{\eta}\| x_{t+1} - x_t\|^2.$$
In practice, $\psi$ is a parameter of our choice, and it is usually set as $\psi = 0.95$. 

Furthermore, the task of identifying $\omega_0$ can be effectively reduced to solving a quadratic equation analytically. This simplification allows for more straightforward implementation and potentially enhances computational efficiency by the following direct computation:
\begin{equation}\label{etai0}
    \omega_{i0} = \min_{\omega_i \in [0,1]} \eta_i,  \text{ such that } a \omega_i^2 + b \omega_i + c \leq 0,
\end{equation}
where 
\begin{equation}\label{eq:quadratic solution}
    \begin{aligned}
        & a = ( \sqrt{f(x^{n+1})} - \tilde{r}_i^{n+1} )^2, \\
        & b = 2\sqrt{f(x^{n+1})} (\tilde{r}_i^{n+1} - \sqrt{f(x^{n+1})} ),\\
        & c = f(x^{n+1}) - (\tilde{r}_i^{n+1})^2 -  \frac{\psi}{\Delta t} ({x_i^{n+1}-x_i^{n}})^2.
    \end{aligned}
\end{equation}
If  $a= 0$, i.e., $\tilde{r}_i^{n+1} = \sqrt{f(x^{n+1})}$,  we set $\omega_{i0}=0$. Otherwise,  the solution to the problem \eqref{etai0} can be written as 
\begin{equation}\label{eq:omega}
    \omega_{i0} = \max\left\{\frac{-b-\sqrt{b^2-4ac}}{2a}, 0 \right\}.
\end{equation}
It can be readily verified that $b^2-4ac \geq 0$ for any $\Delta t$. This condition ensures the real-valued solutions necessary for the stability of the system over the given interval.

In modern deep-learning applications, the parameter count is often extraordinarily high. Applying a single scalar $r$ uniformly across all dimensions is suboptimal. To address this, $r_i$ is computed individually for each dimension. Consequently, the ratio $\rho_i = \frac{r_i}{\sqrt{f(x_t;\xi_t)}}$ acts as an energy-based self-adjusting mechanism for the learning rate $\eta$ within each respective dimension. A value of $\rho_i$ less than $1$ indicates that the current $\eta$ is too large for optimal training progress in that direction for the upcoming step. This triggers an automatic modification of $\eta$ to $\rho_i\eta$, thereby adjusting the update distance for $x$ in that dimension to a more suitable scale.

The Stochastic Gradient \textit{Vector Auxiliary Variable} (VAV) dynamics, also known as Stochastic Gradient \textit{Elementwise Relaxed Scalar Auxiliary Variable} (ERSAV) dynamics, are, for $i=1,2,\ldots,n$:
\begin{subequations}\label{equation:ERSAV}
\begin{align} 
&\frac{x_{t+1,i} - x_{t,i}}{\eta} + \frac{\tilde{r}_{t+1,i}}{\sqrt{f(x_t;\xi_t)}} G_{t,i} = 0,  \\ &\frac{\tilde{r}_{t+1,i} - r_{t,i}}{\eta} = \frac{1}{2 \sqrt{f(x_t;\xi_t)}} G_{t,i} \frac{x_{t+1,1} - x_{t,i}}{\eta} \\ 
&r_{t+1,i} = \omega_{0,i} \tilde{r}_{t+1,i} + (1-\omega_{0,i}) \sqrt{f(x_{t+1};\xi_{t+1})}
& 
\end{align}
\end{subequations}
As before, for a given $\psi \in (0,1),~\omega_{0,i}$ is the smallest number in the interval $[0,1]$ such that
$$(r_{t+1,i})^2 - (\tilde{r}_{t+1,i})^2 \le \frac{\psi}{\eta}( x_{t+1,i} - x_{t,i})^2,$$
for $i=1,2,\ldots,n$. $\omega_{0,i}$ can be computed through \eqref{eq:omega}.
\section{Discussion of the Algorithm} \label{sec:Algorithm}

\subsection{Baseline Algorithm}
The proposed algorithm represents an energy-based self-adaptive learning rate gradient descent methodology. This is achieved by incorporating a multiplier, $\rho = \frac{r_i}{\sqrt{f(x_t)}}$, to the learning rate, effectively introducing an adaptive regularizer into the algorithm. Such an adjustment is anticipated to facilitate faster convergence and reduce the variance of both training and test losses. These phenomena will be further explored and detailed in the experimental section of this paper. A pseudocode representation of the proposed method is succinctly summarized in Algorithm $1$.

\begin{algorithm}[h]
    \caption{VAV Method}
    \label{alg:example}
\begin{algorithmic}
\State {\bfseries Input:} Initial guess $x_0$, \\
                                 Default learning rate $\eta$, \\
                                 $\psi\in(0,1)$ (default $0.95$), \\
                                 Maximum iteration $N$.\\
                                 A non-negative constant number $c$ (default $0$)
\Repeat
\State Set n = 0\\
\State Initialize $r^0 = \sqrt{f(x^0) + c}(1,1,\cdots,1)$.
\While{$n \leq N$}\\
\State {$g_i^n = \frac{\partial f(x^n)}{\partial x_i}$ for $i=1,\cdots,m$}\\
\State {$\tilde{{r}}_i^{n+1} = (1 + 2\eta(\frac{g_n}{f(x^n)+c})^2)^{-1}r_i^n$ for $i=1,\cdots,m$}\\
\State {Update $x_i^{n+1} = x_i^n - \eta\frac{\tilde{{r}}_i^{n+1}}{\sqrt{f(x^n)+c}}g_i^n$}\\
\State {Compute $\omega_i$ by equation \eqref{eq:omega}}\\
\State {Update $r_i^{n+1} = \omega_i\tilde{{r}}_i^{n+1} + (1-\omega_i)\sqrt{f(x^{n+1})+c}$}\\
\State {Update $n = n+1$}\\
\EndWhile
\Until{$n=N$}
\end{algorithmic}
\end{algorithm}
\subsection{Cost of the Proposed Method}
In the baseline model of the proposed algorithm, $r$ and $\omega_0$ are intermediate variables computed at each training step, which slightly increases the computational cost. However, the calculation of these variables requires only the loss value $f(x_t;\xi_t)$ and the gradient of the mini-batch $g_t$ at the current step, both of which are already determined during each iteration. Consequently, the proposed method primarily involves a few linear operations, such as matrix multiplications. Notably, it eschews the need for back-tracking or line searches, and it requires no additional evaluations of function values or gradients throughout the training process. Moreover, the computational demand of the algorithm does not scale with the size of the dataset, which makes it an efficient solution for large-scale applications.
\subsection{The Role of Parameters in the Proposed Algorithm}
\begin{itemize}
    \item \textbf{r and the ratio $\rho$:} As previously discussed, $r$ serves as an approximation of $\sqrt{f(x_t)}$ and it adheres to the dissipative law. The ratio $\rho$ quantifies the discrepancy between $r$ and $f(x_t)$. In all conducted experiments, including those under mini-batch settings and within highly non-convex optimization contexts, $r$ consistently remains below the actual training loss. A significant deviation of $\rho$ from one indicates a high gradient norm, prompting an automatic reduction in the learning rate through multiplication by $\rho$. Conversely, when $\rho$ approaches 1, it suggests that the approximation of the loss of the next step is precise, indicating that the current learning rate is adequate. This mechanism is designed to achieve improved convergence towards the minimum, optimizing the efficiency and effectiveness of the training process.

    \item \textbf{c:} The constant $c$ is defined as a positive number specifically introduced to mitigate computational errors that may arise when the training loss $f(x_t)$ approaches zero—a scenario that can occur due to floating-point imprecision in digital computations. To account for this, $r$ is recalibrated as the approximation of $\sqrt{f(x_t)+c}$. It is important to note that the selection of $c$ influences the behavior of the algorithm: a large value for $c$ eventually guides the algorithm toward Stochastic Gradient Descent (SGD) in the latter stages of training. This convergence is indicated by the ratio $\rho$, which approaches one as $f$ nears zero, thereby aligning the algorithm's performance with that of standard SGD under low-loss conditions.
\end{itemize}
\subsection{Compatibility with other Optimization Techniques}
The proposed method automatically adjusts the learning rate. It is also compatible with other common training techniques such as adding regularizers or momentum. This method is straightforward to implement and can also be applied to other optimization methods, such as Adam.
\subsection{A Direct Application of $r:$ An energy-based self-adaptive scheduler}
When training a model like neural networks, intuitively decreasing the learning rate as the training process nears the global minimum is crucial to mitigate the risk of overshooting. Popular methods such as the linear scheduler or CosineAnnealingLR are commonly implemented in practical applications. However, a significant limitation of these approaches is to predefine hyperparameters and to set up rules prior to training. These preset configurations are often not directly related to the specific tasks of the current training process and finding optimal values for these hyperparameters can be challenging.

The proposed method introduces monotonically decreasing metric $r$ to measure the distance to the minimum. Besides, the algorithm illustrates good stability with a large learning rate because of the self-adjustable property. This allows the user to initiate the training with a relatively high default learning rate, facilitating rapid early-stage convergence. As training progresses and the metric $d=\sqrt{r^2-c}$ falls below the default learning rate, the constant learning rate can be replaced by $d$. Consequently, the learning rate decreases automatically as the training approaches the minimum, directly correlating the adjustment to the training dynamics without the need for predefined schedules.

This strategy does not require any additional computation of new function values or gradients, thereby incurring no extra computational costs. It illustrates faster convergence and improved performance in various training tasks, which will be demonstrated in the experimental section.

\section{Theoretical Analysis} \label{sec:theorem}
This section provides a theoretical analysis of the stochastic gradient Scalar Auxiliary Variable (SAV) algorithm. This analysis will provide intuition about the empirical performance of the proposed method outlined in Algorithm $1$, including its stability, adaptive convergence, and the use of the scalar variable~$r_t$ to adjust the learning rate. Note that the analysis in this section and Appendix~\ref{appendix:theoretical-analysis} can be expanded to cover the vector-based (VAV) and relaxed cases, though this would involve more technical details as described in the final remark of Appendix~\ref{appendix:theoretical-analysis}.

We start by showing that the SAV numerical scheme is stable. In other words, in expected value, SAV is unconditionally energy dissipative.
\begin{theorem} \label{thm:energy-dissipation-main}
The SAV algorithm is unconditionally energy dissipative, that is
$$\mathbb{E}[|r_{t+1}|^2] - \mathbb{E}[| r_{t}|^2] = -\mathbb{E}[|r_{t+1} - r_t|^2]  - \frac{1}{\eta} \|x_{t+1} - x_t\|^2 \le 0.$$
\end{theorem}
We provide the proof of the above theorem in Appendix~\ref{appendix:SG-SAV-stable}. 

We now proceed to study the convergence of the stochastic gradient SAV numerical scheme. To this end, we need to establish first the following six assumptions.
\begin{enumerate}
\item[(A.1)] \textbf{Independent random vectors:} The random vectors $\xi_t,~t=0,1,2,\ldots$, are independent of each other and also of $x_t$.
\item[(A.2)] \textbf{Unbiased gradients and function values:} For each time $t$, the stochastic gradient and the stochastic function value, $G_t$, and $f(x_t,\xi_t)$, are unbiased estimates of $\nabla f(x_t)$ and $f(x_t)$, i.e., $\mathbb{E}_{\xi_t}[G_t] = \nabla f(x_t)$ and $\mathbb{E}_{\xi_t}[f(x_t,\xi_t)] = f(x_t).$
\item[(A.3)] \textbf{Bounded variance:} For a fixed constant $\sigma_0$, the variance of the gradient $G(x_t,\xi_t)$, at any time $t$, satisfies:
    $$
    \mathbb{E}_{\xi_k}[\|G(x_t,;\xi_t) - \nabla f(x_t) \|^2] \le\sigma_0^2.
    $$
\item[(A.4)] \textbf{Bounded function values:} For fixed constants $a$ and $B$, the function $f$ satisfies $0 < a \le f(\cdot)+c \le B$.
\item[(A.5)] \textbf{Bounded gradients:} For a fixed constant $\gamma$, the gradient $\nabla f(x_t)$, at any time $t$, satisfies $\|\nabla f(x) \|^2 \le \gamma^2.$
\item[(A.6)] \textbf{$L$-smoothness:} The function $f$ is $L$-smooth, i.e., has Lipschitz continuous gradients. That is, for every $x,x' \in \mathbb{R}^n$, $\|\nabla f(x) - \nabla f(x') \| \le L \|x-x'\|$.
\end{enumerate}
The rate of convergence for stochastic gradient SAV is then given in the following theorem.
\begin{theorem} \label{thm:SG-SAV-convergence-main}
Let Assumptions (A-1)-(A-6) hold. Then for SAV we have that with probability at least $1-\delta$,  
$$
\min_{t \in [T-1]} \|\nabla f(x_t)\|^{2}  \le \left( \frac{\sqrt{B}}{T}\mathbb{E} \left[ \frac{1}{r_T} \right]\right) \frac{C_F}{\delta^{3/2}},
$$
where 
$$
C_F = \frac{f(x_0) - f^*}{\eta} +  \frac{\gamma^2}{ \sqrt{a}}\sqrt{f(x_0) + c} + \frac{4\sqrt{f(x_0) + c}B}{\eta \sqrt{a}} + \frac{2}{\sqrt{a}}\sqrt{T} \sigma_o \sqrt{2\frac{(f(x_0) + c)B}{\eta}} + \frac{L (f(x_0) + c) B}{a}.
$$
\end{theorem}
The complete proof of Theorem~\ref{thm:SG-SAV-convergence-main} can be found in the Appendix. Also, note that if we let $C_N:= C_F \left(\sqrt{B}/T \mathbb{E}([1/r_T]\right)$, then the above result has the following probabilistic form:
$$\mathbb{P} \left( \min_{t \in [T-1]} \|\nabla f(x_t)\|^2 \ge \frac{C_N}{\delta^{3/2}} \right) \le \delta.$$
The above results clearly show that with a very low probability the proposed bound is violated. Note that this is a consequence of the stochastic nature of our algorithm. 

We will now concentrate on analyzing the main challenge associated with the proof. Note that we start the proof by applying the well-known expected descent lemma:
\begin{align*} 
\mathbb{E}_{\xi_t}[f(x_{t+1})] & \le f(x_t) + \mathbb{E}_{\xi_t} \left[\langle \nabla f(x_t), x_{t+1} - x_t\rangle + \frac{L}{2} \|x_{t+1} - x_t\|^2 \right] \\
& = f(x_t) + \underbrace{\mathbb{E}_{\xi_t} \left[ \left\langle \nabla f(x_t), -\eta \rho_t G_t \right\rangle  \right]}_{=A} + \underbrace{\frac{L}{2} \eta^2 \mathbb{E}_{\xi_t} \left[ \rho_t^2\|G_t\|^2 \right]}_{=B},    
\end{align*}
where $\rho_t = \frac{r_{t+1}}{\sqrt{\tilde{f}_t + c}}$ is the scalar multiplicative factor of the learning rate, $\mathbb{E}_{\xi_t}[\cdot] $ is the expectation with respect to $\xi_t$ conditioned on previous values $\xi_0, \dots, \xi_{t-1}$, and $\tilde{f}_t$ is a simplified notation for $f(x_t, \xi_t)$.

In traditional stochastic gradient descent (SGD), the corresponding term $A$ is $-\eta \|\nabla f(x_t) \|^2$, which is negative - a desirable property. However, in SAV, $A$ does not admit a clear form due to the correlation between $\rho_t$ and $G_t$, making the proof of SAV more challenging compared to SGD. 

To address this challenge, we follow the approach presented in \cite{ward2020adagrad,wang2023convergence} and approximate $\rho_t$ with a surrogate (in our case $\rho_{t-1}$). By using the surrogate, the above inequality takes the form:
\begin{align*} 
\mathbb{E}_{\xi_t}[f(x_{t+1})]  \le f(x_t) + \underbrace{(-\eta) \rho_{t-1} \|\nabla f(x_t) \|^2}_{=A_1}  + \underbrace{\mathbb{E}_{\xi_t} [\left\langle \nabla f(x_t), \eta(\rho_{t-1} - \rho_{t}) G_t \right\rangle]}_{=A_2} + \underbrace{\frac{L}{2} \eta^2 \mathbb{E}_{\xi_t} \left[\rho_t^2\|G_t\|^2 \right]}_{=B}.
\end{align*}
Unlike SGD, the above inequality includes an additional \textit{error term} $A_2$. We will establish a bound for the error term $A_2$ in the supplementary material. Once this bound is established, the convergence proof of Theorem~\ref{thm:SG-SAV-convergence-main} follows naturally, as detailed in Appendix~\ref{appendix:SG-SAV-convergence}. However, this proof still depends on the existence of a positive lower bound for the scalar auxiliary variable, i.e., $r_t > 0$ for all $t$. We will establish this lower bound in the following theorem.
\begin{theorem} \label{thm:r-strictly-positive-main}
Let Assumptions (A.1)-(A.6) hold. In addition, let the variance of the stochastic function value be bounded by~$\sigma_f^2$. Then, if we select the learning rate~$\eta >0$ such that
$$
C_1 \sqrt{\eta} + C_2 \eta + C_3 \le 0,
$$
where 
\begin{align*}
C_1 &= F(x_0)\sqrt{T \left(\frac{\gamma^2}{8a^3}\sigma_f^2 + \frac{1}{2a} \sigma_0^2 \right)},~C_2 = \frac{L_F  F^2(x_0)}{2},~\text{and } C_3 = \frac{1}{2 \sqrt{a}} \sigma_f - \frac{\sqrt{a}}{2}.
        \end{align*}
        Then, we have the following bound:
        $$\mathbb{E}[r_T] \ge \frac{\sqrt{a}}{2} > 0.$$
\end{theorem}
We provide the details of these proofs in the appendix.

\section{Numerical Experiments} \label{sec:Experiments}

\subsection{Toy Example: Rosenbrock function}

The Rosenbrock function, also known as Rosenbrock's valley, is a popular test problem for optimization algorithms. It is a non-convex function designed to evaluate the performance of the optimization techniques. The function is 
$$ f(x, y) = (a - x)^2 + b(y - x^2)^2 $$
with $ a=1, b=100 $, with global minimum at $(x, y)=(1, 1)$ where the function value is zero. We test SGD and the proposed method on this problem with starting point $(-2, -2)$. The outcomes are summarized in Table $1$.

\begin{table}[ht]
\centering
\caption{Result of Rosenbrock Function}
\label{my-table}
\begin{tabular}{ccccc}
\hline
         & learning rate   & Converge or not & Iterations & End Point         \\ \hline
SGD      & 0.01            & No              & nan        & nan               \\
SGD      & 0.005           & Yes             & 15,000      & (0.9846, 0.9693)  \\
VAV    & 0.04            & Yes             & 15,000      & (0.9964, 0.9931)  \\
VAV    & 0.005           & Yes             & 15,000      & (0.9843, 0.9688)  \\ \hline
\end{tabular}
\end{table}

\subsection{Regression Task}

In the domain of machine learning, regression tasks constitute a crucial area of application, particularly within the scope of scientific computing research. One of the recently favored approaches for solving Partial Differential Equations (PDEs) is the Physics-Informed Neural Network (PINN). In this paper, we apply a PINN to model the 1D Burgers' equation and 1D Allen-Cahn equation.

\subsubsection{Burgers' Equation}
In the first part of our regression task, we apply PINN to model the 1D Burgers' equation, a common nonlinear PDE used to describe fluid dynamics and shock waves. The equation is given as follows:

$$ \frac{\partial u}{\partial t} + u\frac{\partial u}{\partial x} = \nu\frac{\partial^2u}{\partial x^2}, \nu=0.01 $$

The initial condition is defined as:

$$ u(x, 0) = sin(\pi x) $$

The spatial domain for our study is defined from $-1$ to $1$ and the temporal domain from $0$ to $0.5$. The training process utilizes $10,000$ collocation points distributed across these domains. Our model architecture is a fully connected neural network with eight hidden layers, each containing $20$ neurons. For evaluating regression performance, test loss serves as the primary metric. The outcomes of the training and test losses are detailed in Figure $1$.

\begin{figure}[ht]
  \centering
  \includegraphics[width=1\textwidth]{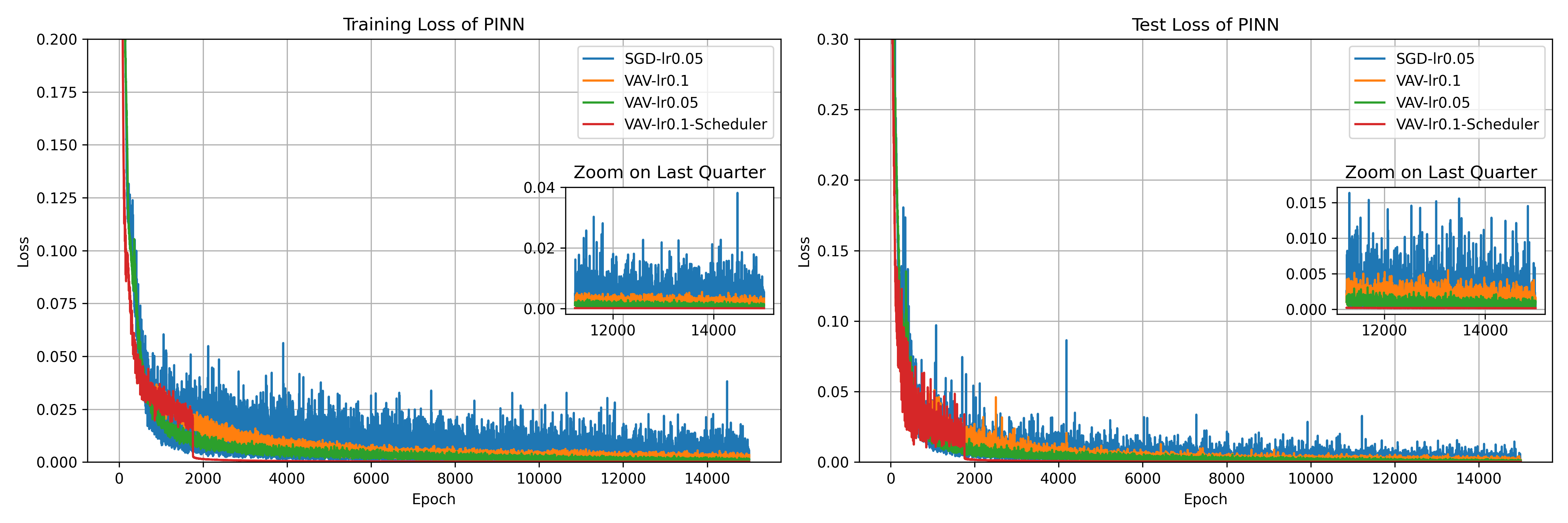}
  \caption{Training (left) and Testing Loss (right) of PINN on Burger's Equation.}
  \label{fig:Burger}
\end{figure}

\subsubsection{Allen-Cahn Equation}
For a second example, we consider the 1D Allen-Cahn equation, another important nonlinear PDE that describes phase separation processes in multi-component alloy systems. The equation is given by:

$$ \frac{\partial u}{\partial t} = \epsilon \frac{\partial^2u}{\partial x^2} + u - u^3, \epsilon = 0.1 $$

The initial condition for this problem is still:

$$ u(x, 0) = sin(\pi x)$$

Similar to the Burgers' equation, the spatial domain is defined from $-1$ to $1$ and the temporal domain is set from $0$ to $2$. We apply the PINN model to solve this equation with the same neural network architecture: eight hidden layers with 20 neurons each. For training, we utilize 2500 collocation points distributed uniformly across the space-time domain.

The model is trained using the same loss functions as in the previous section, where the test loss serves as the primary metric for evaluating regression performance. Figure $2$ provides a comparison of the training and test losses for the Allen-Cahn equation.

\begin{figure}[ht]
  \centering
  \includegraphics[width=1\textwidth]{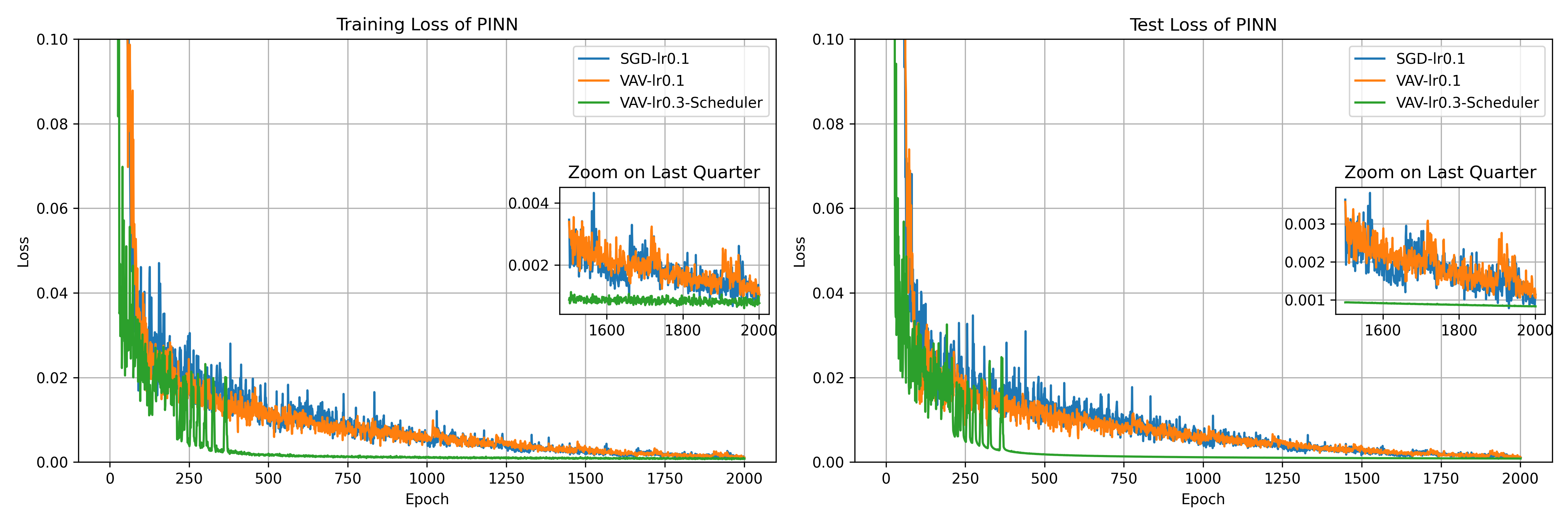}
  \caption{Training (left) and Testing Loss (right) of PINN on Allen-Cahn Equation.}
  \label{fig:Burger}
\end{figure}

\subsection{Classification Task}
We evaluated the performance of our novel optimization algorithm on classification tasks using the CIFAR-10 and CIFAR-100 datasets. CIFAR-10 consists of $60,000$ 32x32 color images in $10$ classes, with $6,000$ images per class, while CIFAR-100 is similar but includes $100$ classes, each containing $600$ images. Both datasets are widely recognized benchmarks in the machine learning community. We utilized the ResNet-50 architecture in \cite{he2015deep}, which contains approximately $25$ million parameters, as the test model. Our experiments were conducted over $200$ epochs with a batch size of $256$, using an initial learning rate of 0.3. We incorporated a popular learning rate scheduling technique that reduces the learning rate by a factor of $10$ at the 150th epoch for SGD and compared it to our proposed scheduler. The comparison results are presented in Figures $3$ and $4$.

\begin{figure}[ht]
  \centering
  \includegraphics[width=1\textwidth]{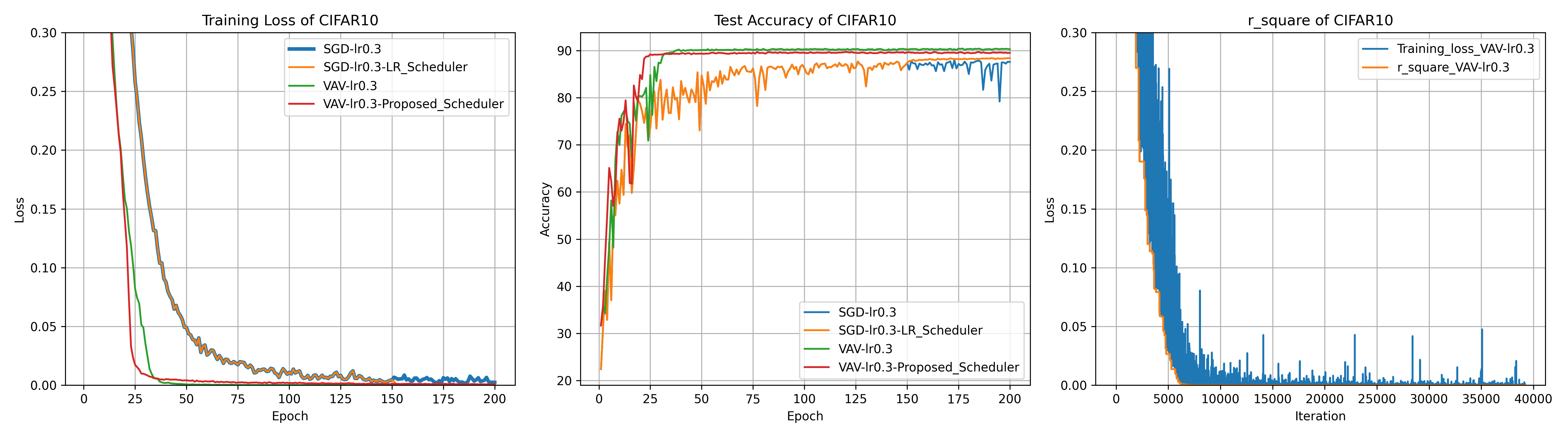}
  \caption{Training Loss per Epoch (Left), Test Accuracy (Center) and r$\_$square per Iteration (Right) on CIFAR10 dataset.}
  \label{fig:Burger}
\end{figure}

\begin{figure}[ht]
  \centering
  \includegraphics[width=1\textwidth]{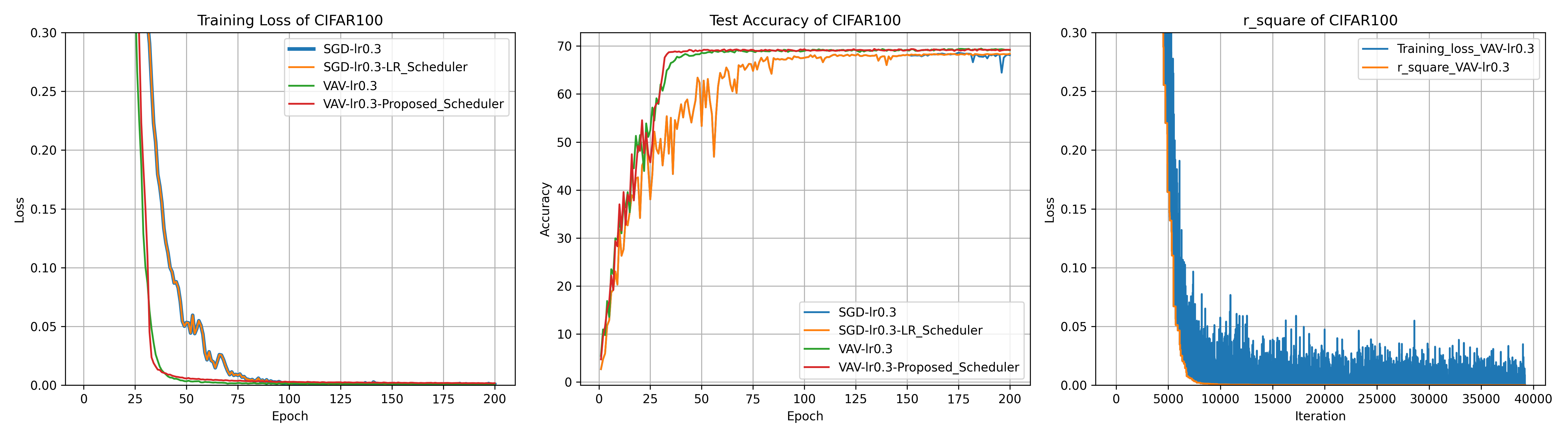}
  \caption{Training Loss per Epoch (Left), Test Accuracy (Center) and r$\_$square per Iteration (Right) on CIFAR100 dataset.}
  \label{fig:Burger}
\end{figure}

\section{Discussion about the numerical experiments} \label{sec:Discussion}

\subsection{Interpreting of the result}

\begin{itemize}
    \item \textbf{Performance:} In this study, the proposed optimization method consistently outperformed SGD across all testing problems. Specifically, in the Rosenbrock function, our method approached a point closer to the global minimum than SGD (0.005 is the highest tolerance learning rate of SGD). In the regression task, the method not only converged more rapidly but also achieved the lowest test error, as depicted by the red line in the figure. In classification tasks, it attained the highest accuracy.

    The proposed scheduler generally enhanced performance in regression tasks, but its effectiveness was not as marked in classification tasks. Although it still outperformed the commonly used MultistepLR, the benefits of applying any scheduler varied depending on different tasks.

    It is important to clarify that we do not claim VAV method guarantees convergence to the global minimum or achieves a superior local minimum compared to SGD. Such assertions would overstate the capabilities of any learning rate adjustment technique. Properly fine-tuned SGD can still yield excellent results, though this falls outside the scope of the current discussion. The objective of this paper is not to establish new state-of-the-art results but to compare the effectiveness of these two methods under identical conditions. Overall, The experiments suggest that with the same settings and almost no additional cost, the VAV method provides a better result compared to SGD.
    
    \item \textbf{Stability on larger learning rate:} In our investigations with toy examples and regression tasks, the proposed optimization method demonstrates its capability to converge to favorable outcomes under significantly larger learning rates, settings in which SGD would typically diverge. This phenomenon is attributed to the self-adaptive properties of the ratio $\rho$. The ability to employ larger learning rates confers several advantages, including the potential for faster exploration of the solution space. This can be particularly beneficial in complex landscapes where smaller learning rates may lead to slow progress or be trapped in suboptimal minima.

    \item \textbf{Variance of Loss:} In all experiments, the variance in both training and testing loss when using the VAV method is notably smaller compared to that observed with SGD. This reduced variance or fluctuation in the training loss often correlates with improved generalizability of the model, as consistent performance during training tends to predict more reliable behavior on unseen data. This may indicate a more reliable final result when deploying the model in real-world applications.
    
    \item \textbf{$r$ as an Empirical Lower Bound:} The last plot in both Figures $2$ and $3$ presents the value of $r^2$ per iteration compared to the training loss. While rigorous proof remains elusive, under the stochastic settings where the loss function varies with each iteration because of the mini-batch, the value of $r^2$ stays close to the training loss and acts as an empirical lower bound. This behavior provides a valuable indicator of the training status and the model's proximity to the optimal minimum. The ability to track this lower bound offers several advantages: it enhances the predictability of the optimization process, aids in diagnosing the training dynamics, and can potentially inform adjustments to improve convergence efficiency. 
\end{itemize}

\subsection{Important Considerations for Implementing the Proposed Method}

\begin{itemize}
    \item \textbf{Numerical Precision Error when Loss is close to zero:} Theorem $4.3$ shows that $r$ has a positive lower bound. This property is critical as it ensures that $r$ does not diminish to zero or become negative as long as the loss function has a positive lower bound. However, deviation might occur in practice when the loss value becomes very close to zero. We believe it is because of the numerical inaccuracies inherent in computational processes by the computer. To resolve this problem, we propose in Section 3 a positive constant $c$ within the algorithm. This adjustment prevents $r$ from approaching zero too closely. It is important for users to note that while adding a larger $c$ value can safeguard against numerical errors, it also accelerates the algorithm's convergence towards the behavior of Stochastic Gradient Descent in the latter stages of training. This adjustment should be carefully calibrated to balance stability with the desired convergence characteristics.
    
    \item \textbf{Batch Size:} A relatively larger batch size is recommended when utilizing the proposed method, particularly when the scheduler is also employed. This is because of the observation that $r$ acts as an empirical lower bound of the training loss. When a batch size is too small, it often results in significant fluctuations or variations in the training loss, which can cause $r$ to decrease quickly. Such a rapid decrease in $r$ may correspondingly cause the learning rate to diminish too fast to a small value, potentially decelerating the training process. To mitigate this issue, users have two possible strategies: incorporating the positive constant $c$, as discussed in Section 3, or using a larger batch size. These adjustments can help maintain a more stable and efficient training trajectory, ensuring the robust performance of the algorithm under varied training dynamics.    
\end{itemize}

\section{Conclusion} \label{sec:Conclusion}

A novel energy-based self-adjustable learning rate optimization method is introduced that dynamically adapts the learning rate based on the training process itself. The VAV method offers improvements in stability and convergence speed over traditional Stochastic Gradient Descent, particularly when handling larger learning rates. The introduced metric $r$ serves as an empirical lower bound for the training loss, guiding the learning rate adjustments to enhance training performance without the need for manual tuning or complex scheduling. We encourage further exploration and application of this method in various machine-learning contexts.

\section*{Acknowledgments}
GL would like to thank the support of National Science Foundation (DMS-2053746, DMS-2134209, ECCS-2328241, CBET-2347401 and OAC-2311848), and U.S.~Department of Energy (DOE) Office of Science Advanced Scientific Computing Research program DE-SC0023161, the Uncertainty Quantification for Multifidelity Operator Learning (MOLUcQ) project (Project No. 81739), and DOE–Fusion Energy Science, under grant number: DE-SC0024583.

\bibliographystyle{plain}
\bibliography{refs}

\begin{thebibliography}{10}

\bibitem{brown1989some}
Andrew~A Brown and Michael~C Bartholomew-Biggs.
\newblock Some effective methods for unconstrained optimization based on the solution of systems of ordinary differential equations.
\newblock {\em Journal of Optimization Theory and Applications}, 62:211--224, 1989.

\bibitem{defazio2023learningratefree}
Aaron Defazio and Konstantin Mishchenko.
\newblock Learning-rate-free learning by d-adaptation, 2023.

\bibitem{defossez2020simple}
Alexandre D{\'e}fossez, L{\'e}on Bottou, Francis Bach, and Nicolas Usunier.
\newblock A simple convergence proof of adam and adagrad.
\newblock {\em arXiv preprint arXiv:2003.02395}, 2020.

\bibitem{duchi2011adaptive}
John Duchi, Elad Hazan, and Yoram Singer.
\newblock Adaptive subgradient methods for online learning and stochastic optimization.
\newblock {\em Journal of machine learning research}, 12(7), 2011.

\bibitem{faw2022power}
Matthew Faw, Isidoros Tziotis, Constantine Caramanis, Aryan Mokhtari, Sanjay Shakkottai, and Rachel Ward.
\newblock The power of adaptivity in sgd: Self-tuning step sizes with unbounded gradients and affine variance.
\newblock In {\em Conference on Learning Theory}, pages 313--355. PMLR, 2022.

\bibitem{gadat2022asymptotic}
S{\'e}bastien Gadat and Ioana Gavra.
\newblock Asymptotic study of stochastic adaptive algorithms in non-convex landscape.
\newblock {\em Journal of Machine Learning Research}, 23(228):1--54, 2022.

\bibitem{he2015deep}
Kaiming He, Xiangyu Zhang, Shaoqing Ren, and Jian Sun.
\newblock Deep residual learning for image recognition, 2015.

\bibitem{jiang2022improving}
Maosheng Jiang, Zengyan Zhang, and Jia Zhao.
\newblock Improving the accuracy and consistency of the scalar auxiliary variable (sav) method with relaxation.
\newblock {\em Journal of Computational Physics}, 456:110954, 2022.

\bibitem{kavis2022high}
Ali Kavis, Kfir~Yehuda Levy, and Volkan Cevher.
\newblock High probability bounds for a class of nonconvex algorithms with adagrad stepsize.
\newblock {\em arXiv preprint arXiv:2204.02833}, 2022.

\bibitem{kingma2017adam}
Diederik~P. Kingma and Jimmy Ba.
\newblock Adam: A method for stochastic optimization, 2017.

\bibitem{li2019convergence}
Xiaoyu Li and Francesco Orabona.
\newblock On the convergence of stochastic gradient descent with adaptive stepsizes.
\newblock In {\em The 22nd international conference on artificial intelligence and statistics}, pages 983--992. PMLR, 2019.

\bibitem{li2020high}
Xiaoyu Li and Francesco Orabona.
\newblock A high probability analysis of adaptive sgd with momentum.
\newblock {\em arXiv preprint arXiv:2007.14294}, 2020.

\bibitem{liu2020aegd}
Hailiang Liu and Xuping Tian.
\newblock Aegd: Adaptive gradient descent with energy.
\newblock {\em Numerical Algebra, Control and Optimization}, 2023.

\bibitem{liu2023efficient}
Xinyu Liu, Jie Shen, and Xiangxiong Zhang.
\newblock An efficient and robust scalar auxialiary variable based algorithm for discrete gradient systems arising from optimizations.
\newblock {\em SIAM Journal on Scientific Computing}, 45(5):A2304--A2324, 2023.

\bibitem{mcmahan2010adaptive}
H~Brendan McMahan and Matthew Streeter.
\newblock Adaptive bound optimization for online convex optimization.
\newblock {\em arXiv preprint arXiv:1002.4908}, 2010.

\bibitem{saupe1989discrete}
Dietmar Saupe.
\newblock Discrete versus continuous newton’s method: A case study.
\newblock {\em Newton’s Method and Dynamical Systems}, pages 59--80, 1989.

\bibitem{shen2018convergence}
Jie Shen and Jie Xu.
\newblock Convergence and error analysis for the scalar auxiliary variable (sav) schemes to gradient flows.
\newblock {\em SIAM Journal on Numerical Analysis}, 56(5):2895--2912, 2018.

\bibitem{shen2018scalar}
Jie Shen, Jie Xu, and Jiang Yang.
\newblock The scalar auxiliary variable (sav) approach for gradient flows.
\newblock {\em Journal of Computational Physics}, 353:407--416, 2018.

\bibitem{shen2019new}
Jie Shen, Jie Xu, and Jiang Yang.
\newblock A new class of efficient and robust energy stable schemes for gradient flows.
\newblock {\em SIAM Review}, 61(3):474--506, 2019.

\bibitem{wang2023convergence}
Bohan Wang, Huishuai Zhang, Zhiming Ma, and Wei Chen.
\newblock Convergence of adagrad for non-convex objectives: Simple proofs and relaxed assumptions.
\newblock In {\em The Thirty Sixth Annual Conference on Learning Theory}, pages 161--190. PMLR, 2023.

\bibitem{ward2020adagrad}
Rachel Ward, Xiaoxia Wu, and Leon Bottou.
\newblock Adagrad stepsizes: Sharp convergence over nonconvex landscapes.
\newblock {\em Journal of Machine Learning Research}, 21(219):1--30, 2020.

\bibitem{zeiler2012adadelta}
Matthew~D. Zeiler.
\newblock Adadelta: An adaptive learning rate method, 2012.

\bibitem{zhang2022numerical}
Jiahao Zhang.
\newblock {NUMERICAL METHOD BASED NEURAL NETWORK AND ITS APPLICATION IN SCIENTIFIC COMPUTING, OPERATOR LEARNING AND OPTIMIZATION PROBLEM}.
\newblock {\em PhD thesis}, 7 2022.

\bibitem{zhang2024energy}
Jiahao Zhang, Shiheng Zhang, Jie Shen, and Guang Lin.
\newblock Energy-dissipative evolutionary deep operator neural networks.
\newblock {\em Journal of Computational Physics}, 498:112638, 2024.

\bibitem{zhao2022stochastic}
Minda Zhao, Zehua Lai, and Lek-Heng Lim.
\newblock Stochastic steffensen method.
\newblock {\em arXiv preprint arXiv:2211.15310}, 2022.

\bibitem{zou2018weighted}
Fangyu Zou, Li~Shen, Zequn Jie, Ju~Sun, and Wei Liu.
\newblock Weighted adagrad with unified momentum.
\newblock {\em arXiv preprint arXiv:1808.03408}, 2, 2018.

\bibitem{zufiria1990application}
Pedro~J Zufiria and Ramesh~S Guttalu.
\newblock On an application of dynamical systems theory to determine all the zeros of a vector function.
\newblock {\em Journal of mathematical analysis and applications}, 152(1):269--295, 1990.

\end{thebibliography}

\appendix
\newpage
\section{Theoretical Analysis} \label{appendix:theoretical-analysis}
We consider the unconstrained optimization problem:
$$
\min_{x \in \mathbb{R}^n}~f(x),
$$
where the function $f:\mathbb{R}^n \to \mathbb{R}$ is continuously differentiable and bounded from below, i.e., $f^* = \inf_{x \in \mathbb{R}^n}f(x) > -c$.

This appendix presents a theoretical analysis of the Stochastic Gradient Scalar Auxiliary Variable (SG-SAV) algorithm, which addresses the above optimization problem. The updated equations for the SG-SAV algorithm are:
\begin{subequations}
\begin{align}
x_{t+1} &= x_t -\eta\frac{ r_{t+1}}{\sqrt{f(x_t;\xi_t) + c}} G_t \label{eq:state} \\ r_{t+1} &= r_t+ \frac{1}{2 \sqrt{f(x_t;\xi_t) + c}} \left\langle G_t,x_{t+1} - x_t\right\rangle. \label{eq:aux-var}
\end{align}    
\end{subequations}
Here, $r_t \in \mathbb{R}_{\ge 0}$ is the scalar auxiliary variable that approximates the energy, $\eta >0$ is the learning rate, $G_t$ is a simplified notation for the stochastic gradient $G(x_t,\xi_t)$, and $\xi_t$ is an independent random variable resulting from mini-batch learning.
\subsection{The SG-SAV Numerical Scheme is Stable} \label{appendix:SG-SAV-stable}
We start by showing that the SG-SAV numerical scheme, defined by equations \eqref{eq:state} - \eqref{eq:aux-var}, is stable. In other words, in expected value, it is unconditionally energy dissipative.
\begin{theorem} \label{thm:energy-dissipation}
The SG-SAV algorithm \eqref{eq:state} - \eqref{eq:aux-var} is unconditionally energy dissipative, that is
$$\mathbb{E}[|r_{t+1}|^2] - \mathbb{E}[| r_{t}|^2] = -\mathbb{E}[|r_{t+1} - r_t|^2]  - \frac{1}{\eta} \|x_{t+1} - x_t\|^2 \le 0.$$
\end{theorem}
\begin{proof}
By multiplying equation \eqref{eq:state} by $(x_{t+1} - x_{t})$, we obtain:
\begin{align} \label{eq:mult-by-state}
\frac{1}{\eta}\|x_{t+1} - x_{t}\|^2  + \frac{r_{t+1}}{\sqrt{f(x_t;\xi_t) + c} } \langle G_{t},x_{t+1} - x_{t} \rangle = 0
\end{align}
Similarly, we multiply \eqref{eq:aux-var} by $2 r_{t+1} $ and obtain
\begin{align} \label{eq:mult-by-r}
2(r_{t+1})^2 - 2r_{t+1}r_{t} - \frac{r_{t+1}}{ \sqrt{f(x_t;\xi_t) + c}} \langle G_{t},x_{t+1} - x_{t} \rangle = 0.
\end{align}
By adding \eqref{eq:mult-by-state} and \eqref{eq:mult-by-r} and taking the expected value $\mathbb{E} [\cdot]$, we complete the proof:
$$\mathbb{E}[|r_{t+1}|^2] - \mathbb{E}[| r_{t}|^2] = -\mathbb{E}[|r_{t+1} - r_t|^2]  - \frac{1}{\eta} \|x_{t+1} - x_t\|^2 \le 0.$$
\end{proof}
\begin{remark}
By using the same techniques employed in the previous proof, it is straightforward to demonstrate that both the element-wise and relaxed versions of the SG-SAV algorithm are numerically stable. In other words, they are also unconditionally energy dissipative.
\end{remark}
\subsection{Convergence of the SG-SAV Numerical Scheme} \label{appendix:SG-SAV-convergence}
In this section, we provide detailed proof of the SG-SAV algorithm's convergence. This convergence result draws inspiration from various findings in adaptive gradient methods such as AdaGrad. AdaGrad was proposed simultaneously by \cite{duchi2011adaptive} and \cite{mcmahan2010adaptive} and its convergence proof over non-convex landscapes has been explored in several studies \cite{ward2020adagrad,li2019convergence,zou2018weighted,li2020high,defossez2020simple,gadat2022asymptotic,kavis2022high,faw2022power}. However, this convergence proof presents additional challenges compared to adaptive gradient methods, which we will address next. We will begin by establishing the following assumptions.
\begin{enumerate}
\item[(A.1)] \textbf{Independent random vectors:} The random vectors $\xi_t,~t=0,1,2,\ldots$, are independent of each other and also of $x_t$.
\item[(A.2)] \textbf{Unbiased gradients and function values:} For each time $t$, the stochastic gradient and the stochastic function value, $G_t$, and $f(x_t,\xi_t)$, are unbiased estimates of $\nabla f(x_t)$ and $f(x_t)$, i.e., $\mathbb{E}_{\xi_t}[G_t] = \nabla f(x_t)$ and $\mathbb{E}_{\xi_t}[f(x_t,\xi_t)] = f(x_t).$
\item[(A.3)] \textbf{Bounded variance:} For a fixed constant $\sigma_0$, the variance of the gradient $G(x_t,\xi_t)$, at any time $t$, satisfies:
    $$
    \mathbb{E}_{\xi_k}[\|G(x_t,;\xi_t) - \nabla f(x_t) \|^2] \le\sigma_0^2.
    $$
\item[(A.4)] \textbf{Bounded function values:} For fixed constants $a$ and $B$, the function $f$ satisfies $0 < a \le f(\cdot)+c \le B$.
\item[(A.5)] \textbf{Bounded gradients:} For a fixed constant $\gamma$, the gradient $\nabla f(x_t)$, at any time $t$, satisfies $\|\nabla f(x) \|^2 \le \gamma^2.$
\item[(A.6)] \textbf{$L$-smoothness:} The function $f$ is $L$-smooth, i.e., has Lipschitz continuous gradients. That is, for every $x,x' \in \mathbb{R}^n$, $\|\nabla f(x) - \nabla f(x') \| \le L \|x-x'\|$.
\end{enumerate}
In the remainder of this section, we'll provide a detailed convergence analysis of SG-SAV. The rate of convergence for SG-SAV is given in the following theorem.
\begin{theorem} \label{thm:SG-SAV-convergence}
Let Assumptions (A-1)-(A-6) hold. Then for SG-SAV we have that with probability at least $1-\delta$,  
$$
\min_{t \in [T-1]} \|\nabla f(x_t)\|^{2}  \le \left( \frac{\sqrt{B}}{T}\mathbb{E} \left[ \frac{1}{r_T} \right]\right) \frac{C_F}{\delta^{3/2}},
$$
where 
$$
C_F = \frac{f(x_0) - f^*}{\eta} +  \frac{\gamma^2}{ \sqrt{a}}\sqrt{f(x_0) + c} + \frac{4\sqrt{f(x_0) + c}B}{\eta \sqrt{a}} + \frac{2}{\sqrt{a}}\sqrt{T} \sigma_o \sqrt{2\frac{(f(x_0) + c)B}{\eta}} + \frac{L (f(x_0) + c) B}{a}.
$$
\end{theorem}
\begin{proof}
We start the proof with the well-known expected descent lemma:
\begin{align} 
\mathbb{E}_{\xi_t}[f(x_{t+1})] & \le f(x_t) + \mathbb{E}_{\xi_t} \left[\langle \nabla f(x_t), x_{t+1} - x_t\rangle + \frac{L}{2} \|x_{t+1} - x_t\|^2 \right] \nonumber \\
& = f(x_t) + \underbrace{\mathbb{E}_{\xi_t} \left[ \left\langle \nabla f(x_t), -\eta \rho_t G_t \right\rangle  \right]}_{=A} + \underbrace{\frac{L}{2} \eta^2 \mathbb{E}_{\xi_t} \left[ \rho_t^2\|G_t\|^2 \right]}_{=B}, \label{eq:descent-lemma}   
\end{align}
where $\rho_t = \frac{r_{t+1}}{\sqrt{\tilde{f}_t + c}}$ is the scalar multiplicative factor defined in the paper, $\mathbb{E}_{\xi_t}[\cdot] $ is the expectation with respect to $\xi_t$ conditioned on previous values $\xi_0, \dots, \xi_{t-1}$, and $\tilde{f}_t$ is a simplified notation for $f(x_t, \xi_t)$.

As discussed in Section~\ref{sec:theorem}, the main challenge in this proof is establishing a bound for $A$. Unlike the traditional SGD, this bound lacks a clear form due to the correlation between $\rho_t$ and $G_t$. To address this challenge, we follow the approach presented in \cite{ward2020adagrad,wang2023convergence} and approximate $\rho_t$ with the surrogate $\rho_{t-1}$. 

Now, we will focus on $A$, which can be decomposed as follows:
\begin{align} 
A &= \mathbb{E}_{\xi_t} [\left\langle \nabla f(x_t), -\eta \rho_{t-1} G_t \right\rangle] + \mathbb{E}_{\xi_t} [\left\langle \nabla f(x_t), \eta(\rho_{t-1} - \rho_{t}) G_t \right\rangle] \nonumber \\
&= -\eta \rho_{t-1} \|\nabla f(x_t) \|^2  + \mathbb{E}_{\xi_t} [\left\langle \nabla f(x_t), \eta(\rho_{t-1} - \rho_{t}) G_t \right\rangle]. \label{eq:A} 
\end{align}
The last term is an \textit{error term} arising from the difference between $\rho_{t}$ and the surrogate $\rho_{t-1}$. If we plug \eqref{eq:A} into \eqref{eq:descent-lemma}, we obtain:
\begin{align} \label{eq:descent-lemma-2}
\mathbb{E}_{\xi_t}[f(x_{t+1})]  \le f(x_t) + \underbrace{(-\eta) \rho_{t-1} \|\nabla f(x_t) \|^2}_{=A_1}  + \underbrace{\mathbb{E}_{\xi_t} [\left\langle \nabla f(x_t), \eta(\rho_{t-1} - \rho_{t}) G_t \right\rangle]}_{=A_2} + \underbrace{\frac{L}{2} \eta^2 \mathbb{E}_{\xi_t} \left[\rho_t^2\|G_t\|^2 \right]}_{=B}  
\end{align}
The remainder of this convergence proof is split into two stages. In Stage I, which we will discuss next, we establish a bound for the error term $A_2$.
\subsubsection*{Stage I: Bounding the error term $A_2$}
For simplicity, we will use the following notation: $\tilde{F}_t = \sqrt{\tilde{f}_t + c}$. Then, based on the definition of $\rho_t$, we get:
\begin{align*} 
&|\eta(\rho_{t-1} -\rho_t) \langle \nabla f(x_t), G_t \rangle | = \left|\eta \left(\frac{r_t}{\tilde{F}_{t-1}} - \frac{r_{t+1}}{\tilde{F}_t} \right) \langle \nabla f(x_t), G_t \rangle\right| \\ &=\left|\frac{\eta}{\tilde{F}_{t-1}}(r_t - r_{t+1}) \langle \nabla f(x_t), G_t \rangle + \eta \left(\frac{1}{\tilde{F}_{t-1}} - \frac{1}{\tilde{F}_t} \right) r_{t+1}\langle \nabla f(x_t), G_t \rangle \right| \\ &=\left|\frac{\eta}{\tilde{F}_{t-1}}(r_t - r_{t+1}) \langle \nabla f(x_t), G_t \rangle + \eta \left(\frac{1}{\tilde{F}_{t-1}} - \frac{1}{\tilde{F}_t} \right) r_{t+1} \|G_t\|^2 + \eta \left(\frac{1}{\tilde{F}_{t-1}} - \frac{1}{\tilde{F}_t} \right) r_{t+1}\langle \nabla f(x_t) - G_t, G_t \rangle \right| \\ &\overset{(i)}{\le}  \left|\frac{\eta}{\tilde{F}_{t-1}}| \underbrace{(r_t - r_{t+1})}_{\ge 0} \|\nabla f(x_t)\| \|G_t\| +  \left|\frac{\eta}{\tilde{F}_{t-1}} - \frac{\eta}{\tilde{F}_t} \right| r_{t+1}\| G_t\|^2 + \left|\frac{\eta}{\tilde{F}_{t-1}} - \frac{\eta}{\tilde{F}_t} \right| r_{t+1} |\langle \nabla f(x_t) -G_t, G_t \rangle \right| \\ &\overset{(ii)}{\le} \frac{\eta \gamma^2}{\sqrt{a}} (r_t - r_{t+1}) + \frac{2 \eta}{\sqrt{a}} r_{t+1} \|G_t\|^2 + \frac{2 \eta}{\sqrt{a}} r_{t+1} |\langle \nabla f(x_t) -G_t, G_t \rangle|, \end{align*}
where inequality $(i)$ follows from Cauchy-Schwartz and inequality $(ii)$ from Assumptions (A.4) and (A.5). Then, using the above inequality, the error term $A_2$ is bounded as follows:
$$
A_2 \le \frac{\eta \gamma^2}{\sqrt{a}} \mathbb{E}_{\xi_t}[(r_t - r_{t+1})] + \frac{2 \eta}{\sqrt{a}} \mathbb{E}_{\xi_t} \left[ r_{t+1} \|G_t\|^2 \right] + \frac{2 \eta}{\sqrt{a}} \mathbb{E}_{\xi_t} \left[ r_{t+1} |\langle \nabla f(x_t) -G_t, G_t \rangle| \right].
$$
By inserting the error term $A_2$ bound into equation \eqref{eq:descent-lemma-2}, we obtain:
\begin{align*} 
\mathbb{E}_{\xi_t}[f(x_{t+1})]  \le f(x_t) -\eta \rho_{t-1} \|\nabla f(x_t) \|^2  &+ \frac{\eta \gamma^2}{\sqrt{a}} \mathbb{E}_{\xi_t}[(r_t - r_{t+1})] + \frac{2 \eta}{\sqrt{a}} \mathbb{E}_{\xi_t} \left[ r_{t+1} \|G_t\|^2 \right] + \\& \frac{2 \eta}{\sqrt{a}} \mathbb{E}_{\xi_t} \left[ r_{t+1} |\langle \nabla f(x_t) -G_t, G_t \rangle| \right] + \frac{L}{2} \eta^2 \mathbb{E}_{\xi_t} \left[\rho_t^2\|G_t\|^2 \right].
\end{align*}
We use the law of total expectation, taking the expectation of both sides from the above inequality with respect to $\xi_0, \dots, \xi_{t-1}$. This results in the following recursion:
\begin{align*} 
\eta \mathbb{E}[\rho_{t-1} \|\nabla f(x_t)\|^2]  \le \mathbb{E}[f(x_t))] - \mathbb{E}[f(x_{t+1})]   &+  \frac{\eta \gamma^2}{\sqrt{a}} \mathbb{E}[r_t - r_{t+1}] + \frac{2 \eta}{\sqrt{a}} \mathbb{E} \left[ r_{t+1} \|G_t\|^2 \right] + \\& \frac{2 \eta}{\sqrt{a}} \mathbb{E} \left[ r_{t+1} |\langle \nabla f(x_t) -G_t, G_t \rangle| \right] + \frac{L}{2} \eta^2 \mathbb{E} \left[\rho_t^2\|G_t\|^2 \right].  
\end{align*}
By summing the above recursion over $t$ from $0$ to $T-1$, we obtain the following:
\begin{align} \label{eq:SG-SAV-recursive-bound}
        \eta \sum_{t=0}^{T-1} \mathbb{E}[\rho_{t-1} \|\nabla f(x_t)\|^2]  \le f(x_0) - f^*   &+  \frac{\eta \gamma^2}{\sqrt{a}} \mathbb{E}[r_0 - r_T] + \frac{2 \eta}{\sqrt{a}} \underbrace{\mathbb{E} \left[ \sum_{t=0}^{T-1} r_{t+1} \|G_t\|^2 \right]}_{=C} \\ &+ \frac{2 \eta}{\sqrt{a}} \underbrace{\mathbb{E} \left[ \sum_{t=0}^{T-1} r_{t+1} |\langle \nabla f(x_t) -G_t, G_t \rangle| \right]}_{=D} + \frac{L}{2} \eta^2 \underbrace{\mathbb{E} \left[ \sum_{t=0}^{T-1}  \rho_{t}^2 \| G_t\|^2 \right]}_{=E}. \nonumber    \end{align}
We will now establish bounds for the terms $C$, $D$, and $E$ from the preceding inequality. We will begin with $C$, using the dynamics of the scalar auxiliary variable~$r$, which we can rewrite as follows:
$$r_t - r_{t+1} = \eta \frac{\rho_t}{2 \tilde{F}_t} \|G_t \|^2.$$
Then, by summing the above equation over $t$ from $0$ to $T-1$, we obtain
\begin{align*} 
r_0  &\ge \frac{\eta}{2} \sum_{t=0}^{T-1}  \frac{r_{t+1}}{\tilde{f}_t + c} \| G_t\|^2  \\ &\overset{(i)}{\ge} \frac{\eta}{2B} \sum_{t=0}^{T-1}  r_{t+1} \| G_t\|^2, 
\end{align*}
where inequality $(i)$ follows from Assumption (A.4). Then, by taking the expectation of the above inequality, we arrive to the bound of $C$:
$$C = \mathbb{E} \left[ \sum_{t=0}^{T-1} r_{t+1} \|G_t\|^2 \right] \le \frac{2 F(x_0) B }{\eta} = \frac{2 \sqrt{f(x_0) + c} B}{\eta}.$$
Next, we will establish a bound for $E$ as this bound will also assist us in determining a bound for $D$. 
\begin{align*} 
\sum_{t=0}^{T-1} \rho^2_t \|G_t\|^2 =  \sum_{t=0}^{T-1} \frac{r_{t+1}^2}{\tilde{f}_t + c}  \|G_t\|^2  &\le \frac{r_0}{a} \sum_{t=0}^{T-1} r_{t+1} \|G_t\|^2 
\end{align*}
By taking expectation to the above inequality and using the bound for $C$, we obtain the following bound for $E$:
$$E = \mathbb{E} \left[\sum_{t=0}^{T-1} \rho^2_t \|G_t\|^2 \right] \le \frac{r_0}{a} \mathbb{E} \left[ \sum_{t=0}^{T-1} r_{t+1} \|G_t\|^2 \right] \le \frac{2 (f(x_0) + c) B}{a\eta}.$$
Finally, we establish a bound for $D$ as follows:
\begin{align*} 
D = \mathbb{E} \left[ \sum_{t=0}^{T-1} r_{t+1} |\langle \nabla f_t -G_t, G_t \rangle| \right] &= \mathbb{E} \left[ \sum_{t=0}^{T-1}  |\langle \nabla f_t -G_t, r_{t+1} G_t \rangle| \right] \\ &\overset{(i)}{\le} \left(\mathbb{E} \left[ \sum_{t=0}^{T-1}  \|\nabla f_t -G_t \|^2 \right]\right)^{1/2} \left(\mathbb{E} \left[ \sum_{t=0}^{T-1} r_{t+1}^2 \|G_t \|^2 \right]\right)^{1/2} \\ &\overset{(ii)}{\le} \sqrt{T} \sigma_o \sqrt{\frac{2(f_0 + c)B}{\eta}}.
\end{align*} 
In the above, inequality $(i)$ follows from the Cauchy-Schwarz inequality, and inequality $(ii)$ follows from Assumption (A.3) and the bound for $E$.

Plugging the bounds for $C$, $D$, and $E$ into equation \eqref{eq:SG-SAV-recursive-bound} and using Assumption (A.4) yields:
\begin{align} \label{eq:SG-SAV-recursive-bound-2}
        \frac{\eta}{\sqrt{B}} \sum_{t=0}^{T-1} \mathbb{E}[r_t \|\nabla f(x_t)\|^2]  \le f(x_0) - f^*   &+  \frac{\eta \gamma^2}{\sqrt{a}} \sqrt{f(x_0) + c} +  \frac{4 \sqrt{f(x_0) + c} B}{\sqrt{a}}  \\ &+ \frac{2 \eta}{\sqrt{a}} \sqrt{T} \sigma_o \sqrt{\frac{2(f(x_0) + c)B}{\eta}} +  \frac{ L \eta (f(x_0) + c) B}{a}. \nonumber    
\end{align}
We now proceed to Stage II of the proof where we convert the bound of $\sum_{t=0}^{T-1} r_t \|\nabla f(x_t)\|^2$ into the bound of $\sum_{t=0}^{T-1} \|\nabla f(x_t)\|^2$.
\subsubsection*{Stage II: Bounding $\sum_{t=0}^{T-1} \|\nabla f(x_t)\|^2$}
To proceed with Stage II of the proof, we apply Holder's inequality to the left-hand side of inequality \eqref{eq:SG-SAV-recursive-bound-2}, 
$$
\frac{\mathbb{E}|XY|}{(\mathbb{E}|Y|^3)^{1/3}} \le (\mathbb{E} |X|^{3/2})^{2/3}.
$$
with 
$$
X = (r_t \|\nabla f(x_t)\|^2)^{2/3} \text{ and } Y = (1/r_t)^{2/3}.
$$
to obtain
$$
\mathbb{E} \left[ r_t \| \nabla f(x_t) \|^2 \right] \ge \frac{\left(\mathbb{E} \|\nabla f(x_t)\|^{4/3} \right)^{3/2}}{\left( \mathbb{E}[1/r_t^{2/3}] \right)^{3/2}} \ge \frac{\left(\mathbb{E} \|\nabla f(x_t)\|^{4/3} \right)^{3/2}}{ \mathbb{E}[1/r_T]}.
$$
Note that the above holds since $r_t >0$ for any $t=0,1,\ldots$ (see Theorem~\ref{thm:r-strictly-positive}). 
We then arrive to the inequality
\begin{align*} 
\frac{T \min_{t \in [T-1]} \left(\mathbb{E} \left[ \|\nabla f(x_t)\|^{4/3} \right] \right)^{3/2}}{\sqrt{B} \mathbb{E}[1/r_T]} \le \frac{f(x_0) - f^*}{\eta}   &+  \frac{\gamma^2}{ \sqrt{a}}\sqrt{f(x_0) + c} + \frac{4\sqrt{f(x_0) + c}B}{\eta \sqrt{a}} \\ &+ \frac{2}{\sqrt{a}}\sqrt{T} \sigma_o \sqrt{2\frac{(f(x_0) + c)B}{\eta}} + \frac{L (f(x_0) + c) B}{a}.
\end{align*}
By multiplying both sides of the above inequality by $\frac{\sqrt{B}}{T}\mathbb{E} \left[ \frac{1}{r_T} \right]$, we get:
$$
\min_{t \in [T-1]} \left(\mathbb{E} \left[ \|\nabla f(x_t)\|^{4/3} \right] \right)^{3/2} \le \underbrace{\left( \frac{\sqrt{B}}{T}\mathbb{E} \left[ \frac{1}{r_T} \right]\right) C_F}_{=:C_N},
$$
where
\begin{align*} 
C_F = \frac{f(x_0) - f^*}{\eta} +  \frac{\gamma^2}{ \sqrt{a}}\sqrt{f(x_0) + c} + \frac{4\sqrt{f(x_0) + c}B}{\eta \sqrt{a}} + \frac{2}{\sqrt{a}}\sqrt{T} \sigma_o \sqrt{2\frac{(f(x_0) + c)B}{\eta}} + \frac{L (f(x_0) + c) B}{a}.
\end{align*}
Finally, we complete the proof by applying Markov’s inequality:
\begin{align*} 
\mathbb{P} \left(\min_{t \in [T-1]} \|\nabla f(x_t) \|^2 \ge \frac{C_N}{\delta^{3/2}} \right) &= \mathbb{P} \left(\min_{t \in [T-1]} \left(\|\nabla f(x_t)\|^2\right)^{2/3} \ge \left( \frac{C_N}{\delta^{3/2}}\right)^{2/3}\right)  \\ &\le \delta \frac{\mathbb{E} \left[ \min_{t \in [T-1]} \|\nabla f(x_t) \|^{4/3} \right]}{C_N^{2/3}} \\ &\le \delta.
\end{align*}
\end{proof}
\subsection{A Positive Lower Bound for the Scalar Auxiliary Variable $r_T$} \label{appendix:lower-bound-r}
We conclude this theoretical analysis of the SG-SAV algorithm by demonstrating that the scalar auxiliary variable $r_t$ can be bounded from below by a positive constant for any $t$. 
\begin{theorem} \label{thm:r-strictly-positive}
Let Assumptions (A.1)-(A.6) hold. In addition, let the variance of the stochastic function value be bounded by~$\sigma_f^2$. Then, if we select the learning rate~$\eta >0$ such that
$$
C_1 \sqrt{\eta} + C_2 \eta + C_3 \le 0,
$$
where 
\begin{align*}
C_1 &= F(x_0)\sqrt{T \left(\frac{\gamma^2}{8a^3}\sigma_f^2 + \frac{1}{2a} \sigma_0^2 \right)},~C_2 = \frac{L_F  F^2(x_0)}{2},~\text{and } C_3 = \frac{1}{2 \sqrt{a}} \sigma_f - \frac{\sqrt{a}}{2}.
        \end{align*}
        Then, we have the following bound:
        $$\mathbb{E}[r_T] \ge \frac{\sqrt{a}}{2} > 0.$$
\end{theorem}
\begin{proof}
To simplify our notation, we let $F_t = \sqrt{f(x_t) + c}$ and $\tilde{F}_t=\sqrt{f(x_t;\xi_t) + c}$. We first show that $F$ is $L_F$-smooth. That is, for any  $x,y \in \{x_t\}_{t=0}^T$, we have
\begin{align*}
\|\nabla F(x) - \nabla F(y)\| &= \left\|\frac{\nabla f(x)}{2F(x)} - \frac{\nabla f(y)}{2F(y)} \right\| \\ &= \frac{1}{2} \left\|\frac{\nabla f(x)(F(y) - F(x))}{F(x)F(y)} + \frac{\nabla f(x) - \nabla f(y)}{F(y)} \right\| \\ & \overset{(i)}{\le} \frac{\gamma}{2 F(x^*)}|F(y) - F(x)| + \frac{1}{2F(x^*)} \|\nabla f(x) - \nabla f(y)\| \\ & \overset{(ii)}{\le} \underbrace{\frac{1}{2F(x^*)} \left( 1 + \frac{\gamma^2}{2F(x^*)^2}\right)}_{=:L_F} \|x-y\|.
\end{align*}
In the above, inequality $(i)$ follows from Assumption (A.5), $F(x^*)^2 \le F(x)F(y)$, and $F(x^*) \le F(x)$. Similarly, inequality $(ii)$ follows from $|F(y) - F(x)| \le \frac{\gamma}{2F(x^*)} \| x-y\|.$

By the $L_F$-smoothness of $F$, we have: 
\begin{align*} 
F(x_{t+1}) - F(x_t) &\le \langle \nabla F(x_t),x_{t+1} - x_{t} \rangle + \frac{L_F}{2} \|x_{t+1} - x_t\|^2 \\ &= \left\langle \nabla F(x_t) - \frac{G_t}{\tilde{F}_t}, x_{t+1} -x_t \right\rangle  + \left\langle \frac{G_t}{\tilde{F}_t}, x_{t+1} -x_t\right \rangle + \frac{L_F}{2} \|x_{t+1} - x_t\|^2. 
\end{align*}
We now sum over $t$ from $0$ to $T-1$ and take the expectation $\mathbb{E}[\cdot]$:
\begin{align*} 
\mathbb{E}[F(x_{T})] - \mathbb{E}[F(x_0)] \le &\underbrace{\mathbb{E} \left[ \sum_{t=0}^{T-1}\left \langle\nabla F(x_t) - \frac{G_t}{2\tilde{F}_t},x_{t+1} -x_t \right \rangle \right]}_{= A} \\&+ \underbrace{\mathbb{E} \left[\sum_{t=0}^{T-1}\left \langle \frac{G_t}{2\tilde{F}_t}, x_{t+1} -x_t \right \rangle \right]}_{= B} +  \underbrace{\frac{L_F}{2} \mathbb{E} \left[\sum_{t=0}^{T-1} \|x_{t+1} - x_t\|^2 \right]}_{= C}.  
\end{align*}
We proceed to bound $A$, $B$, and $C$ next. We first establish a bound for $A$:
\begin{align*} 
A &= \mathbb{E} \left[ \sum_{t=0}^{T-1}\left \langle \nabla F(x_t) - \frac{G_t}{2\tilde{F}_t},x_{t+1} -x_t \right \rangle \right] \\ & \le \mathbb{E} \left[ \sum_{t=0}^{T-1}\left \| \nabla F(x_t) - \frac{G_t}{2\tilde{F}_t} \right \| \|x_{t+1} -x_t\|\right] \\ &\overset{(i)}{\le} \mathbb{E} \left[ \left( \sum_{t=0}^{T-1}\left \| \nabla F(x_t) - \frac{G_t}{2\tilde{F}_t} \right \|^2 \right)^{1/2} \left(\sum_{t=0}^{T-1}\|x_{t+1} -x_t\|^2  \right)^{1/2} \right]  \\ & \overset{(ii)}{\le} \underbrace{\left( \mathbb{E} \left[ \sum_{t=0}^{T-1}\left \| \nabla F(x_t) - \frac{G_t}{2\tilde{F}_t} \right \|^2\right]  \right)^{1/2}}_{= A_1} \underbrace{ \left( \mathbb{E} \left[\sum_{t=0}^{T-1}\|x_{t+1} -x_t\|^2 \right] \right)^{1/2}}_{= A_2} \\ &\le F(x_0) \sqrt{\eta n T} \sqrt{\frac{\gamma^2}{8a^3}\sigma_f^2 + \frac{1}{2a} \sigma_0^2}.  
\end{align*}
Here, inequality $(i)$ follows from the Cauchy-Schwartz inequality. For inequality $(ii)$, $A_1$ and $A_2$ are bounded as follows.
\begin{align*} 
\left\| \nabla F(x_t) - \frac{G_t}{2\tilde{F}_t} \right \|^2 &=  \left \| \frac{\nabla f(x_t)}{2F(x_t)} - \frac{G_t}{2\tilde{F}_t} \right \|^2 \\ &= \frac{1}{4} \left \| \frac{\nabla f(x_t) (\tilde{F}_t - F(x_t))}{F(x_t)\tilde{F}_t} + \frac{\nabla f(x_t) - G_t}{\tilde{F}_t} \right \|^2 \\ &\le \frac{1}{2} \left \| \frac{\nabla f(x_t) (\tilde{F}_t - F(x_t))}{F(x_t)\tilde{F}_t} \right \|^2 + \frac{1}{2} \left \| \frac{\nabla f(x_t) - G_t}{\tilde{F}_t} \right \|^2 \\ & \overset{(i)}{\le} \frac{\gamma^2}{2a^2} |\tilde{F}_t - F(x_t)|^2 + \frac{1}{2a} \|\nabla f(x_t) - G_t \|^2, 
\end{align*}
where inequality $(i)$ follows from Assumption (A.4). Taking the expectation $\mathbb{E}[\cdot],$ we obtain 
\begin{align*} 
\mathbb{E} \left[ \left\| \nabla F(x_t) - \frac{G_t}{2\tilde{F}_t} \right \|^2 \right] & \le \frac{\gamma^2}{2a^2} \mathbb{E}[|\tilde{F}_t - F(x_t)|^2] + \frac{1}{2a} \mathbb{E}[\|\nabla f(x_t) - G_t \|^2] \\ &\overset{(i)}{\le} \frac{\gamma^2}{8a^3}\sigma_f^2 + \frac{1}{2a} \sigma_0^2,
\end{align*}
where inequality $(i)$ follows from the  bounded variance of the stochastic function and 
\begin{align*} 
\mathbb{E}[|\tilde{F}_t - F(x_t)|^2] &= \mathbb{E} \left[ \left | \frac{f(x_t) - f(x_t ;\xi_t)}{F(x_t) + \tilde{F}_t} \right | \right] \\  &\le \frac{1}{4a^2} \mathbb{E} \left[ | f(x_t) - f(x_t ;\xi_t) |^2 \right] \\  &\le \frac{1}{4a^2} \sigma_f^2.
\end{align*}
Using the above, we establish the following bound for $A_1$:
\begin{align*} 
A_1 &= \left( \mathbb{E} \left[ \sum_{t=0}^{T-1}\left \| \nabla F(x_t) - \frac{G_t}{2\tilde{F}_t} \right \|^2\right]  \right)^{1/2} \le \sqrt{T} \sqrt{\frac{\gamma^2}{8a^3}\sigma_f^2 + \frac{1}{2a} \sigma_0^2}.
\end{align*} 
We establish the bound for $A_2$ using the unconditional energy dissipative property:
\begin{align*} 
A_2 = \left( \mathbb{E} \left[\sum_{t=0}^{T-1}\|x_{t+1} -x_t\|^2 \right] \right)^{1/2} &\le \left( \eta(\mathbb{E}[\|r_0\|^2] - \underbrace{\mathbb{E}[\|r_T\|^2]}_{\ge 0}) \right)^{1/2} \\ &\le \sqrt{\eta} F(x_0).  \end{align*}
We now establish a bound for $B$ as follows:
\begin{align*} B = \mathbb{E} \left[\sum_{t=0}^{T-1}\left \langle \frac{G_t}{2\tilde{F}_t},x_{t+1} -x_t\right \rangle \right] &= \mathbb{E} \left[ \sum_{t=0}^{T-1} (r_{t+1} - r_{t})\right] = \mathbb{E}[r_T] - \mathbb{E}[r_0] \\ &= \mathbb{E}[r_T] - \mathbb{E}[\tilde{F}_0].
\end{align*}
Finally, we bound $C$ as follows:
\begin{align*} 
C = \frac{L_F}{2} \mathbb{E} \left[\sum_{t=0}^{T-1} \|x_{t+1} - x_t\|^2 \right] \le \frac{L_F \eta F^2(x_0)}{2}.  \end{align*}
Using the bounds for $A$, $B$, and $C$, we have:
\begin{align*} 
\mathbb{E}[F(x_{T})] - \mathbb{E}[F(x_0)] &\le  F(x_0) \sqrt{\eta T} \sqrt{\frac{\gamma^2}{8a^3}\sigma_f^2 + \frac{1}{2a} \sigma_0^2} + \mathbb{E}[r_T] - \mathbb{E}[\tilde{F}_0] + \frac{L_F \eta F^2(x_0)}{2}.   
\end{align*}
By rearranging, we have:
\begin{align*} 
\mathbb{E}[r_T] - \mathbb{E}[\tilde{F}_0] + F(x_0)&\ge  \mathbb{E}[F(x_{T})]  - F(x_0) \sqrt{\eta T} \sqrt{\frac{\gamma^2}{8a^3}\sigma_f^2 + \frac{1}{2a} \sigma_0^2} - \frac{L_F \eta F^2(x_0)}{2} \\ &\ge F(x^*) - F(x_0) \sqrt{\eta T} \sqrt{\frac{\gamma^2}{8a^3}\sigma_f^2 + \frac{1}{2a} \sigma_0^2} - \frac{L_F \eta F^2(x_0)}{2}.
\end{align*}
The right-hand side of the above inequality can be written as
\begin{align*} 
\mathbb{E}[r_T] - \mathbb{E}[\tilde{F}_0] + F(x_0) &= \mathbb{E}[r_T] + \mathbb{E}[F(x_0) - \tilde{F}_0] \\ &\le \mathbb{E}[r_T] + \mathbb{E}[|F(x_0) - \tilde{F}_0|] \\ &\le \mathbb{E}[r_T] + \frac{1}{2 \sqrt{a}} \sigma_f.
\end{align*}
Thus, we obtain the lower bound:
\begin{align*} 
\mathbb{E}[r_T] \ge   F(x^*) - F(x_0) \sqrt{\eta T} \sqrt{\frac{\gamma^2}{8a^3}\sigma_f^2 + \frac{1}{2a} \sigma_0^2} - \frac{L_F \eta  F^2(x_0)}{2} - \frac{1}{2 \sqrt{a}} \sigma_f.
\end{align*}
Note that if we select a learning rate $\eta >0$, such that 
$$C_1 \sqrt{\eta} + C_2 \eta + C_3 \le 0,$$
where 
\begin{align*}
C_1 &= F(x_0)\sqrt{T \left(\frac{\gamma^2}{8a^3}\sigma_f^2 + \frac{1}{2a} \sigma_0^2 \right)},~C_2 = \frac{L_F  F^2(x_0)}{2},~\text{and } C_3 = \frac{1}{2 \sqrt{a}} \sigma_f - \frac{\sqrt{a}}{2},
\end{align*}
then, we conclude the proof by establishing the following bound:
$$\mathbb{E}[r_T] \ge \frac{\sqrt{a}}{2} > 0.$$
\end{proof}
\begin{remark}
\textit{Extending the theoretical result to stochastic gradient Vector Auxiliary Variable (VAV).} The proof for stochastic gradient VAV follows the same steps as the proof for SAV. However, the proof for VAV is more involved as it requires more technical details such as (i) expressing the vector of auxiliary variables $(r_{t,1},\dots,r_{t,n})^\top$ as a diagonal matrix, (ii) using matrix norms, (iii) scaling some of our bounds by the dimension $n$ of $r_t$, and (iv) expressing our results in terms of $\min_{i} r_{T,i}$.
\end{remark}
\section{Experimental Settings and Resources}

    This appendix provides detailed information on the experimental settings and resources utilized in our study to ensure reproducibility.
    
\subsection{Experimental Settings for Regression Task in Section $5.2$}
\begin{itemize}
    \item \textbf{Data Split:} The training set consists of $10000$ collocation points for Burgers' equation and $2500$ for Allen-Cahn equation across the domain. The test set consists of $900$ collocation points in the domain. 
    \item \textbf{Hyperparameters:} Learning rate is specified in Section $5$ and Figure $1,2$. The parameter $c$ introduced in Section $3$ is set to $0$. The number of epochs for Burgers' equation is $15,000$, while the number of epochs for the Allen-Cahn equation is $2,000$. The batch size is selected as $1,000$ for Burger's equation and $128$ for the Allen-Cahn equation.
    \item \textbf{Model Architecture:} Fully connected neural network with eight hidden layers, each containing 20 neurons. The activation function used is Tanh.
    \item \textbf{Optimizer:} SGD and the proposed VAV method.
    \item \textbf{Loss Function:} MSE Loss.
    \item \textbf{Hardware:} A100 GPU in Colab.
    \item \textbf{Software:} Pytorch2.2.1.
\end{itemize}

\subsection{Experimental Settings for Classification Task in Section $5.3$}
\begin{itemize}
    \item \textbf{Data Split:} Training and test set of CIFAR10 and CIFAR100 imported from Torchvision.
    \item \textbf{Hyperparameters:} Learning rate is specified in Section $5$ and Figure $2, 3$. The parameter $c$ introduced in Section $3$ is set to be $0$ for cases without a scheduler and $0.01$ for cases with a scheduler. The number of epochs is $200$. The batch size is selected as $256$.
    \item \textbf{Model Architecture:} Resnet50.
    \item \textbf{Optimizer:} SGD and the proposed VAV method.
    \item \textbf{Loss Function:} CrossEntropy Loss.
    \item \textbf{Hardware:} A100 GPU in Colab.
    \item \textbf{Software:} PyTorch2.2.1.
\end{itemize}

\end{document}